\documentclass[runningheads]{llncs}
\usepackage{amsmath, amsfonts, amssymb, amstext, bbm, bbold, mathtools}
\usepackage[T1]{fontenc}
\usepackage{graphicx}
\usepackage[colorlinks=true,linkcolor=black,anchorcolor=black,citecolor=black,filecolor=black,menucolor=black,runcolor=black,urlcolor=black,pdfencoding=auto]{hyperref}

\usepackage[ruled,vlined]{algorithm2e}
\usepackage{physics}
\usepackage{caption}
\usepackage{subcaption}
\usepackage[dvipsnames,table,usenames]{xcolor}
\usepackage{soul} 
\usepackage{todonotes}
\usepackage{cleveref}
\usepackage{etoolbox}
\usepackage[misc]{ifsym}

\DeclareMathOperator*{\argmax}{arg\,max}
\DeclareMathOperator*{\argmin}{arg\,min}
\def\x{{\bf x}}

\def\A{{\bf A}}

\def\D{{\bf D}}
\def\y{{\bf y}}

\def\I{{\bf I}}

\def\S{{\bf S}}

\def\Rset{\mathbb R}

\def\thetabf{{\boldsymbol \theta}}
\def\alphabf{\boldsymbol{\alpha}}

\def\etabf{\boldsymbol{\eta}}
\def\epsilonbf{\boldsymbol{\epsilon}}
\def\gammabf{\boldsymbol{\gamma}}

\def \ie{\emph{i.e. }}

\begin{document}
%
\title{Hypothesis Transfer in Bandits by Weighted Models}

\titlerunning{Hypothesis Transfer in Bandits by Weighted Models}

\toctitle{Hypothesis~Transfer~in~Bandits~by~Weighted~Models}

\author{Steven Bilaj(\Letter)\inst{1} \and
Sofien Dhouib\inst{1} \and
Setareh Maghsudi\inst{1}}
\authorrunning{S. Bilaj et al.}
\tocauthor{Steven~Bilaj}
%
\institute{Eberhard Karls University of Tübingen, Tübingen, Germany\\\email{\{steven.bilaj, sofiane.dhouib, setareh.maghsudi\}@uni-tuebingen.de}
}
\maketitle        
\begin{abstract}
We consider the problem of contextual multi-armed bandits in the setting of hypothesis transfer learning. That is, we assume having access to a previously learned model on an unobserved set of contexts, and we leverage it in order to accelerate exploration on a new bandit problem. Our transfer strategy is based on a re-weighting scheme for which we show a reduction in the regret over the classic Linear UCB when transfer is desired, while recovering the classic regret rate when the two tasks are unrelated. We further extend this method to an arbitrary amount of source models, where the algorithm decides which model is preferred at each time step. Additionally we discuss an approach where a dynamic convex combination of source models is given in terms of a biased regularization term in the classic LinUCB algorithm. The algorithms and the theoretical analysis of our proposed methods substantiated by empirical evaluations on simulated and real-world data.

\keywords{Multi-Armed Bandits \and Linear Reward Models  \and Recommender Systems \and Transfer Learning.}
\end{abstract}
\section{Introduction}
The \emph{multi-armed bandit} problem (MAB) \cite{thompson1933likelihood,robbins1952some,bush1953stochastic} revolves about maximizing the reward collected by playing actions from a predefined set, with uncertainty and limited information about the observed payoff. At each round, the bandit player chooses an arm according to some rule that balances the exploitation of the currently available knowledge and the exploration of new actions that might have been overlooked while being more rewarding. This is known as the exploration-exploitation trade-off. MAB's find applications in several areas \cite{bouneffouf2020survey}, notably in recommender systems \cite{li2010contextual,zhou2017large,liu2018transferable,labille2021transferable}.
In these applications, the number of actions to choose from can grow very large, and it becomes provably detrimental to the algorithm's performance to ignore any side information provided when playing an action or dependence between the arms \cite{auer2002nonstochastic}. Considering such information defines the \emph{Stochastic Contextual Bandits} \cite{langford2007epoch,li2010contextual,chu2011contextual,abbasi2011improved} setting, where playing an action outputs a context-dependent reward, where a context can correspond to a user's profile and/or the item to recommend in recommender system applications.
Hence, less exploration is required as arms with correlating context vectors share information, thus further reducing uncertainty in the reward estimation. This ultimately led to lower regret bounds and improved performance \cite{abbasi2011improved}. 

While the stochastic contextual bandit problem solves the aforementioned issues, it disregards the possibility of learning from previously trained bandits. For instance, assume a company deploys its services in a new region. Then it would waste the information it has already learned from its previous recommending experience if it is not leveraged to accelerate the recognition of the new users' preferences. Such scenarios have motivated transfer learning for bandits \cite{soare2014multi,liu2018transferable,suk2021self,labille2021transferable}, which rely on the availability of contexts of the previously learned tasks to the current learner. However, regarding a setup where context vectors correspond to items which have been selected by a user, privacy issues are encountered in healthcare applications \cite{stark2019literature,ras2021recommender} for instance, the aim being to recommend a treatment based on a patient's health state. Indeed, accessing the contexts of the previous tasks entails the history of users' previous activities. Moreover, in engineering applications such as scheduling of radio resources \cite{amrallah2020radio}, storage issues \cite{maiti2021multi,liau2018stochastic,xu2021memory} might arise when needing access to the context history of previous tasks. These problems would render algorithms depending on previous tasks' contexts inapplicable.

In this work, we aim to reduce exploration by exploiting knowledge from a previously trained contextual bandit accessible only through its parameters, thus accelerating learning if such model is related to the one at hand, and ultimately decreasing the regret. We extend this idea by including an arbitrary amount of models increasing the likelihood of including useful knowledge.
To summarize our contributions, we propose a variation of the Linear Upper Confidence Bound (\emph{LinUCB}) algorithm, which has access to previously trained models called source models. The knowledge transfer takes place by using an evolving convex combination of sources models and a LinUCB model, called a target model, estimated with the collected data. The combination's weights are updated according to two different weighting update strategies which minimize the required exploration factor and consecutively the upper regret bound, while also taking a lack of information into consideration. 
Our regret bound is at least as good as the classic LinUCB one \cite{abbasi2011improved}, where the improvement depends on the quality of the source models. Moreover, we prove that if the source model used for transfer is not related to our problem, then it will be discarded early on and we recover the LinUCB regret rate. In other words, our algorithm is immune against negative transfer. We test our algorithm on synthetic and real data sets and show experimentally how the overall regret improves on the classic model.\\

The rest of the paper is organized as follows. We discuss related work in \Cref{sec:relatedwork} and formulate our problem in \Cref{sec:problemFormulation}, then we provide and analyse our weighting solution in \Cref{sec:weightedLinBandits}. This is followed by an extension to the case where one has access to more than one trained model in \Cref{sec:multi_source}. Finally, the performance of our algorithm is assessed in \Cref{sec:experiments}.
\section{Related Work}\label{sec:relatedwork}
We hereby discuss two families of contributions related to ours, namely transfer for multi-armed bandits, and hypothesis transfer learning.
\subsubsection{Transfer for MAB's} To the best of our knowledge, tUCB \cite{azar2013sequential} is the first algorithm to tackle transfer in an MAB setting. Given a sequence of bandit problems picked from a finite set, it uses a tensor power method to estimate their parameters in order to transfer knowledge to the task at hand, leading to a substantial improvement over UCB. Regarding the richer contextual MAB setting, MT-LinUCB \cite{soare2014multi} reduces the confidence set of the reward estimator by using knowledge from previous episodes.  More recently, transfer for MAB's has been applied to recommender systems \cite{liu2018transferable,labille2021transferable}, motivated by the cold start problem where a lack of initial information requires more exploration at the cost of higher regret. The TCB algorithm \cite{liu2018transferable} assumes access to correspondence knowledge between the source and target tasks, in addition to contexts, and achieves a regret of $O(d\sqrt{n\log{n}})$ as in the classic LinUCB case, with empirical improvement. The same regret rate holds for the T-LinUCB algorithm \cite{labille2021transferable}, which exploits prior observations to initialize the set of arms, in order to accelerate the training process. 
The main difference of our formulation with respect to the previous ones is that we assume having access only the the preference vectors of the previously learnt tasks, without their associated contexts, which goes in line with the Hypothesis Transfer Learning setting. Even with such a restriction, we keep the LinUCB regret rate and we show that the regret is lower in the case source parameters that are close to those of the task at hand.

\subsubsection{Hypothesis Transfer Learning} Using previously learned models in order to improve learning on a new task defines the hypothesis transfer learning scenario, also known as model reuse or learning from auxiliary classifiers. Some lines of work consider building the predictor of the task at hand as the sum of a source one (possibly a weighted combination of different models) and the one learned from the available data points \cite{yang2007cross,duan2009domain,tommasi2014learning}. Such models were thoroughly analyzed in \cite{kuzborskij2013stability,kuzborskij2017fast,perrot2015theoretical} by providing performance guarantees. The previously mentioned additive form of the learned model was further studied and generalized to a large family of transformation functions in \cite{du2017hypothesis}. In online learning, the pioneering work of \cite{zhao2014online} relies on a convex combination instead of a sum, with adaptive weights. More recently,  the \texttt{Condor} algorithm \cite{zhao2020handling} was proposed and theoretically analyzed to handle the concept drift scenario, relying on biased regularization w.r.t. a convex combination of source models. Our online setting involves transfer with  decisions over a large set of alternatives at each time step, thus it becomes crucial to leverage transfer to improve exploration. To this end, we use a weighting scheme inspired by \cite{zhao2014online} but that relies on exploration terms rather than on how the models approximate the rewards.

\section{Problem Formulation}\label{sec:problemFormulation}
We consider a contextual bandit setting in which at each time $k$, playing an action $a$ from a set $\mathcal A$ results in observing a context vector $\x_{a_k} \in \Rset^d$ assumed to satisfy $\norm{\x_{a_k}}\leq 1$ , in addition to a reward $r(k)$. We further define the matrix induced norm: $\norm{\x}_{\A}\coloneqq\sqrt{\x^T\A\x}$ for any vector $\x\in \Rset^d$ and any matrix $\A\in\Rset^{d\times d}$. The classical case aims to find an estimation $\hat{\thetabf}$ of an optimal bandit parameter $\thetabf^* \in \Rset^d$ which determines the rewards $r$ of each arm with context vector $\x_a$ in a linear fashion $r=\x_a^T\thetabf^*+\epsilon$ up to some $\sigma$-subgaussian noise $\epsilon$. The decision at time $k$ is made according to an upper confidence bound (UCB) associated to $\hat{\thetabf}(k)$:
\begin{equation}
    a_k =\argmax_{a\in\mathcal A} \x^T_a\hat{\thetabf}(k)+\gamma\sqrt{\x_a^T\A^{-1}(k)\x_a},
    \label{classic_linucb}
\end{equation}

where $\gamma>0$ is a hyperparameter estimated through the derivation of the UCB later and $\A(k)\coloneqq\lambda\I_d+\sum_{k'=1}^{k}\x_{a_{k'}}\x_{a_{k'}}^T$. The latter term in the sum \eqref{classic_linucb} represents the exploration term which decreases the more arms are explored. $\hat{\thetabf}(k)$ is computed through regularized least-squares regression with regularization parameter $\lambda>0$: $\hat{\thetabf}(k)=\A^{-1}(k)\D^T(k) \y(k)$,  with $\D(k)=[\x^T_{a_i}]_{i\in\{1,...,k\}}$ and $\y(k)=[r(i)]_{i\in\{1,...,k\}}$ as the concatenation of selected arms' context vectors and corresponding rewards respectively. We alter this decision making approach with the additional use of a previously trained source bandit. Inspired by \cite{zhao2014online}, we transfer knowledge from one linear bandit model to another by a weighting approach. We denote the parameters of the source bandit by $\thetabf_S \in \Rset^d$. The bandit at hand's parameters are then estimated as:
\begin{equation}
    \hat{\thetabf}=\alpha_{S} \thetabf_S + \alpha_{T} \hat{\thetabf}_{T}(k),
\end{equation}
with weights $\alpha_{S},\alpha_{T}\geq 0$ satisfying $\alpha_{S}+\alpha_{T} = 1$. More important is how the exploration term changes and how it affects the classic regret bound. From \cite{abbasi2011improved} we know that the upper bound of the immediate regret in a linear bandit algorithm directly depends on the exploration term of the UCB. We aim to reduce the required exploration with the use of the source bandits knowledge, in order to accelerate the learning process as well as reducing the upper regret bound. For the analysis we consider the pseudo-regret \cite{audibert2009exploration} defined as:

\begin{equation}
    R(n)=n \max_{a \in\mathcal A} \x_{a}^T\theta^* - \sum_{k=1}^n\x_{a_k}^T\thetabf^*.
\end{equation}
Our goal is to prove that this quantity is reduced if the source bandit is related to the one at hand, whereas its rate is not worsened in the opposite case.
\section{Weighted Linear Bandits}\label{sec:weightedLinBandits}
The model we use features dynamic weights, thus at time $k$, we use the following model for our algorithm:

\begin{equation}
    \hat{\thetabf}(k)=\alpha_{S}(k)\thetabf_S+\alpha_{T}(k)\hat{\thetabf}_{T}(k),
\end{equation}
with $\hat{\thetabf}_{T}(k)$ being updated like in the classic LinUCB case \cite{abbasi2011improved} and $\thetabf_S$ remaining constant. To devise an update rules of the weights, we first re-write the new UCB expression as:
\begin{equation}
    \mathrm{UCB}(a)=\x_a^T\left(\alpha_{S}(k)\thetabf_S+\alpha_{T}(k)\hat{\thetabf}_{T}(k)\right)+\left(\alpha_{S}(k)\gamma_S+\alpha_{T}(k)\gamma_T\right)\norm{\x_a}_{\A^{-1}},
    \label{new_LinUCB}
\end{equation}

with $\gamma_S\geq\norm{\thetabf^*-\thetabf_S}_{\A(k)}$ and $\gamma_T\geq\norm{\thetabf^*-\hat{\thetabf}_{T}(k)}_{\A(k)}$ as confidence set bounds for the source bandit and target bandit respectively. We retrieve the classic case by setting $\alpha_{S}(k)$ to zero \ie erasing all influence from the source. The confidence set bound $\gamma_T$ has already been determined in \cite{abbasi2011improved}.\\
As mentioned in \cref{sec:problemFormulation} we aim to reduce the required exploration in order to reduce the upper regret bound. Thus we select the weights such that the exploration term in (\ref{new_LinUCB}) is minimized.

\subsection{Weighting Update Strategies}
We want to determine the weights after each time step such that:

\begin{equation}
    \alpha_{S}, \alpha_{T}=\argmin_{\substack{\alpha_{S}', \alpha_{T}'\geq 0\\\alpha_S' + \alpha_T' = 1}} \alpha_{S}'\gamma_S + \alpha_{T}'\gamma_T.
    \label{hard_update_optimization}
\end{equation}
The above minimization problem is solved for:
\begin{equation}
    \alpha_{S}= \mathbb{1}_{\gamma_S \leq \gamma_T}
    ,\quad \alpha_T = 1-\alpha_S. \label{hard_update}
\end{equation}
This strategy would guarantee an upper regret bound at least as good as the LinUCB bound in \cite{abbasi2011improved} as will be shown in the analysis section later. However, without any knowledge of the relation between source and target tasks, our upper bound on the confidence set of the source bandit is rather loose:
\begin{align*}
    \norm{\thetabf^*-\thetabf_S}_{\A(k)} &= \sqrt{\lambda U^2 + \norm{\D(k)(\thetabf^*-\thetabf_S)}_2^2}
    &\leq\sqrt{4\lambda+\norm{\overline{\y}(k)-\y_{S}(k)}_2^2},
\end{align*}
with $\norm{\theta^*-\theta_S}_2= U$,  $\y_S$ as the concatenation of the source estimated rewards and $\overline{\y}$ as the concatenation of the observed mean rewards for each arm. Naturally after every time step, each entry in $\overline{\y}$ corresponding to the latest pulled arm needs to be updated to their mean value. The mean values are taken in order to cancel out the noise term in the observations. Also, we have $U \leq 2$ in case the vectors show in opposing directions and we additionally assume that $\norm{\thetabf^*},\norm{\thetabf_S}\leq1$. An upper bound on the confidence set $\gamma_T$ of the target bandit has been determined in \cite{abbasi2011improved}:

\begin{equation}
    \gamma_T = \sqrt{d\log(1+\frac{k}{d\lambda})+\log(\frac{1}{\delta^2})}.
\end{equation}

As such, $\gamma_T$ grows with $\sqrt{\log(k)}$ and later on in the analysis we show if $\thetabf_S \neq \thetabf^*$ then an upper bound on $\gamma_S$ grows with at least $\sqrt{k}$. Consequently, in theory there is some point in time where $\gamma_{S}$ will outgrow $\gamma_{T}$, meaning that the source bandit will be discarded. As already mentioned, our estimation of $\gamma_S$ can be loose due to our lack of information on the euclidean distance term $U$, thus we potentially waste a good source bandit with this strategy. Additionally we would only use one bandit at a time this way instead of the span of two bandits for example. Alternatively we can adjust the strategy in \eqref{hard_update_optimization} by adding a regularization term in the form of KL-divergence. By substituting $\alpha_T = 1- \alpha_S$ we get: 

\begin{equation}
    \alpha_{S}(k+1)=\argmin_{\alpha_S\in[0,1]}\left\langle\begin{pmatrix}\alpha_S \\ 1-\alpha_S\end{pmatrix},\begin{pmatrix}\gamma_{S} \\\gamma_{T}\end{pmatrix}\right\rangle + \frac{\mathrm{KL}(\alphabf\Vert\alphabf(k))}{\beta}, \label{eq:min_kl}
\end{equation}

with $\alphabf\coloneqq(\alpha_{S}, 1-\alpha_{S})^T$ being a vector containing both weights. The addition of the KL divergence term forces both weights to stay close to their previous value, where $\beta>0$ is a hyper parameter controlling the importance of the regularization. Problem \eqref{eq:min_kl} is solved for:

\begin{equation}
    \alpha_{S}(k+1) = \frac{1}{1+\frac{1-\alpha_{S}(k)}{\alpha_{S}(k)}\exp(\beta(\gamma_{S} - \gamma_{T}))},
    \label{sigmoid_update}
\end{equation}

which is a softened version of our solution in \eqref{hard_update}, but in this case the source bandit will not be immediately discarded if the upper bound on its confidence set becomes larger than the target bandit's.

\subsection{Analysis}
We are going to analyse how the upper regret bound changes, within our model in comparison to \cite{abbasi2011improved}. All proofs are given in the appendix. First we bound the regret for the hard update approach, not including the KL-divergence term in \eqref{hard_update}:

\begin{theorem}
Let $\{\x_{a_k}\}_{k=1}^N$ be sequence in $\Rset^d$, $U\coloneqq\norm{\thetabf_S-\thetabf^*}$ and $R_T$ be the classic regret bound of the linear model  \cite{abbasi2011improved}. Let $m \coloneqq \min(\kappa,n)$ and $\delta\leq\exp(-2\lambda)$.  Then, with a probability at least $1-\delta$, the regret of the hard update approach for the weighted LinUCB algorithm is bounded as follows:

\begin{equation}
    R(n) \leq U\sqrt{8m d\log(1 +\frac{m}{d\lambda})(\lambda+m)} + R_T(n) - R_T(m)\leq R_T(n)
\end{equation}

with $\kappa$ satisfying:

\begin{equation}
    \kappa = \left\lfloor2\left[d\left(\frac{1}{U^2}-\lambda\right)+\lambda\left(\frac{2}{U^2}-\frac{1}{2}\right)\right]\right\rfloor.
\end{equation}

\end{theorem}

The value for $\kappa$ essentially gives a threshold such that we have $\gamma_S<\gamma_T$ for every $k<\kappa$. As expected, for better sources \ie low values $U$, $\kappa$ increases meaning the source is viable for more time steps. Also notable is how we see an increasing value for $\kappa$ at high dimensional spaces. This is most likely due to the fact, that at higher dimensions the classic algorithm requires more time steps, in order to find a suitable estimation, thus having a larger confidence set bound. In these instances a trained source bandit would be viable early on.
The regret is reduced for lower values of $U$ and the time $\kappa$ at which a source is discarded is extended. For source bandits satisfying $\norm{\thetabf_S-\thetabf^*}_2=2$, we would retrieve the classic regret bound, preventing negative transfer. \\
Next we show what happens in case of a negative transfer for the softmax update strategy, \ie the source does not provide any useful information at all and worsens the regret rate with $\gamma_S > \gamma_T$ at all time steps.

\begin{theorem}
Let \{$\x_{a_k}$\}$_{k=1}^N$ be sequence in $\Rset^d$ and the minimal difference between confidence set bounds given as $\Delta_{\mathrm{min}}=\min_{k \in \{0,...,N\}}(\gamma_S(k)-\gamma_T(k))$, with $\gamma_S > \gamma_T$ for all time steps and the initial target weight denoted by $\alpha_T(0)$. Then with probability of at least $1-\delta$ an upper regret bound $R(n)$ in case of a negative transfer scenario is given by:
\begin{equation}
    R(n) \leq \frac{(1-\alpha_{T}(0))}{e\beta\alpha_{T}(0)(1-\exp(-\beta\Delta_{\mathrm{min}}))}+R_T(n)
\end{equation}
\label{theorem_negative_transfer}
\end{theorem}

Theorem \ref{theorem_negative_transfer} shows that in case of a negative transfer, the upper regret bound is increased by at most a constant term and vanishes in the case of $\beta \xrightarrow{}\infty$ retrieving the hard update rule.

\section{Weighted Linear Bandits with Multiple Sources}\label{sec:multi_source}
Up until now we only used a single source bandit, but our model can easily be extended to an arbitrary amount of different sources. Assuming we have $M$ source bandits $\{\theta_{S,j}\}_{j=1}^M$, we define $\hat\theta$ as:

\begin{equation}
    \hat{\thetabf}=\sum_{j=1}^M\alpha_{S,j}\thetabf_{S,j}+\alpha_T\hat{\thetabf}_T,
\end{equation}

with $\alpha_{S,j}, \alpha_T\geq 0\ \forall 1\leq j\leq M$ and $\alpha_T + \sum_{j=1}^M\alpha_{S,j}=1$. With this each source bandit yields its own confidence set bound $\gamma_{S,j}$. Similarly to (\ref{new_LinUCB}) we retrieve for the UCB with multiple sources:

\begin{equation}
    \mathrm{UCB}(a)=\x_a^T\left(\sum_{j=1}^M\alpha_{S,j}(k)\thetabf_{S,j}+\alpha_{T}(k)\hat{\thetabf}_{T}(k)\right)+\alphabf^T(k)\gammabf\norm{\x_a}_{\A^{-1}(k)},
    \label{mult_linucb}
\end{equation}

with $\alphabf(k)=(\alpha_{S,1}(k),...,\alpha_{S,M}(k),\alpha_T(k))^T$ and $\gammabf=(\gamma_{S,1},...,\gamma_{S,M},\gamma_T)^T$. As for the weight updates the same single source strategies apply \ie the minimization of the exploration term in the UCB function:

\begin{equation}
    \alphabf(k+1)=\argmin_{\alphabf \in \mathfrak P_{M+1}} \alphabf^T(k)\gammabf + \frac{1}{\beta}\mathrm{KL}(\alphabf||\alphabf(k)),
    \label{multi_soft_strat}
\end{equation}

where $\mathfrak P_{M+1}$ is the $(M+1)-$dimensional probability simplex. The solution of the previous problem is:

\begin{equation}
    \alpha_{S,m}(k+1) = \frac{\alpha_{S,m}(k)\exp\left(-\beta \gamma_{S,m}\right)}{\sum_{j=1}^M\alpha_{S,j}(k)\exp(-\beta \gamma_{S,j})+\alpha_T(k)\exp(-\beta\gamma_T)}.
    \label{multi_weight_update}
\end{equation}

This is basically the solution of (\ref{sigmoid_update}) generalized to multiple sources. In the decisions making it favours the bandit with the lowest upper bound $\gamma$ of their confidence set. When we take the limit $\beta\xrightarrow[]{}\infty$ in (\ref{multi_soft_strat}) the KL-divergence term vanishes and we retrieve the hard case: 

\begin{equation}
    \alpha_{S,j}= \mathbb{1}_{\gamma_{S,j}=\min\left(\min_{i}\gamma_{S,i},\gamma_T\right)} 
    \label{hard_multi_update}
\end{equation}

which forces the weights to satisfy $\alpha_{S,m},\alpha_T\in\{0,1\}$ for every source index and for all time steps. Thus decision making is done by selecting one single bandit in each round with the lowest value of their respective confidence set bound $\gamma$. The regret of hard update strategy for multiple sources is given by the following theorem:

\begin{theorem}
Let \{$\x_{a_k}$\}$_{k=1}^N$ be sequence in $\Rset^d$ and $\min_m\norm{\thetabf_{S,m}-\thetabf^*} = U_\mathrm{min}$ and the classic regret bound of the linear model up to time step $n$ given by $R_T(n)$ \cite{abbasi2011improved}. Let $m \coloneqq \min(\kappa,n)$ and $\delta\leq\exp(-2\lambda)$. Then with probability of at least $1-\delta$ the regret of the hard update approach for the weighted LinUCB algorithm with multiple sources is bounded by:

\begin{equation}
    R(n) \leq 4 U_{\mathrm{min}}\sqrt{\kappa d\log(1 + \kappa/(d\lambda))(\lambda+\kappa)}-R_T(m) + R_T(n)\leq R_T(n),
\end{equation}

with $\kappa$ as:

\begin{align*}
    \kappa = \left\lfloor2\left[d\left(\frac{1}{U_\mathrm{min}^2}-\lambda\right)+\lambda\left(\frac{2}{U_\mathrm{min}^2}-\frac{1}{2}\right)\right]\right\rfloor.
\end{align*}

\label{theorem_mult_hard}
\end{theorem}
depending on $U_\mathrm{min}$ the multiple source approach benefits from the additional information as the upper bound corresponds to the best source overall.
In case of the softmax update strategy, we need to show how the regret changes in case of a negative transfer scenario, \ie the confidence set bounds of any source bandit is larger than the target bound at any time.

\begin{theorem}
Let \{$\x_{a_k}$\}$_{k=1}^N$ be sequence in $\Rset^d$, a total of $M$ source bandits being available indexed by $j$ and the minimal difference between confidence set bounds set as $\Delta_{\mathrm{min},j}=\min_{k \in \{0,...,N\}}(\gamma_{S,j}(k)-\gamma_T(k))$ for every source $j$ with $\gamma_{S,j} > \gamma_T$ $\forall j$ at every time step. Additionally the initial target weight is denoted by $\alpha_{T}(0)$. Then with probability $1-\delta$ an upper regret bound $R(n)$ in case of a negative transfer scenario is given by:
\begin{equation}
    R(n) \leq \frac{(1-\alpha_{T}(0))}{e\beta M\alpha_{T}(0)}\sum_{j=1}^M \frac{1}{(1-\exp(-\beta\Delta_{\mathrm{min},j}))}+R_T
\end{equation}
\end{theorem}

In comparison to the single source result, the additional constant is averaged over all sources. Depending on the quality, it can be beneficial to include more source bandits as potentially bad sources would be mitigated.

\begin{algorithm}[H]
\caption{Weighted LinUCB}
Initialize: $\hat{\thetabf}_{T}(0)$ from $\mathcal U([0, 1]^d)$, $\alpha_{S,j}(0)=(1-\alpha_{T}(0))/M = \frac{1}{2M}$, $U_j>0$ $\gamma_{S,j}>0$ $\forall j\in\{1,...,M\}$, $\delta\in[0, 1]$, $\gamma_T>0$, $\lambda>0$, $\beta>0$, ${\bf A}(0) =\lambda\bf I$, ${\bf b}(0)=\bf 0$; \\
\For{$k=0...N$}{
Pull arm $a_k = \argmax_{a} \mathrm{UCB}(a)$ taken from (\ref{mult_linucb});\\
Receive estimated rewards from sources and real rewards: $r_{S,j}(k)|_{j\in\{0,...,M\}}, r(k)$;\\
${\bf A}(k+1) = {\bf A}(k) + {\bf x}_{a_k} {\bf x}_{a_k}^T$;\\
${\bf b}(k+1)= {\bf b}(k) + r(k) {\bf x}_{a_k}$;\\
$\hat{\thetabf}_{T}(k+1)={\bf A}^{-1}(k+1) {\bf b}(k+1)$;\\
Store rewards $r_{S,j}(k)|_{j\in\{0,...,M\}}, r(k)$ in vectors ${\bf y}_{S,j}(k)|_{j\in\{0,...,M\}}, {\bf y}(k)$ respectively;\\
Calculate $\overline{\y}(k)$ from $\y(k)$ such that each entry $r$ corresponding to the latest arm $a_k$ pulled is updated to the mean reward $\overline{r}$ of the respective arm;\\
Update $U_j=\max_{i\in\{0,...,k\}}\frac{\abs{\overline{r}(i)-r_{S,j}(i)}}{\norm{\x_{a_i}}}$ for every $j$;\\
$\gamma_{S,j} = \sqrt{\lambda U_j+\norm{{\bf y}_{S,j}(k)-\overline{\bf y} (k)}}$;\\
$\gamma_T = \sqrt{\lambda}+\sqrt{\log{\frac{\norm{{\bf A}(k)}}{\lambda^d\delta^2}}}$;\\
update source weights $\alpha_{S,j}(k+1)$ according either to softmax rule in (\ref{multi_weight_update}):\\
or to the hard update rule in \eqref{hard_multi_update};\\
update target weight as:\\
$\alpha_{T}(k+1)=1-\sum_{j=1}^M\alpha_{S,j}(k+1)$;}
\end{algorithm}

For the practical implementation we use $\gamma_T = \sqrt{\lambda}+\sqrt{\log{\frac{\norm{{\bf A}(k)}}{\lambda^d\delta^2}}}$ which is also taken from \cite{abbasi2011improved} and gives a tighter confidence set bound on the target estimator. Also we give an estimation for $U_j$ by taking the maximum value of the lower bound induced by the Cauchy-Schwartz inequality $U_j=\norm{\thetabf_{S,j}-\thetabf^*}\geq\max_{i\in\{0,...,k\}}\frac{\abs{\overline{r}(i)-r_{S,j}(i)}}{\norm{\x_{a_i}}}$ at each time step.

\subsection{Biased Regularization}
In \cite{zhao2020handling} a similar approach of model reuse was used in a concept drift scenario for linear classifiers via biased regularization. In \cite{kuzborskij2017fast} the risk generalization analysis for this approach was delivered in a supervised offline learning setting. Their mathematical formulation is stated as following: A classifier is about to be trained given a target training set $(\D, \y)$ and a source hypothesis $\thetabf_{src}$, which is specifically used for a biased regularization term. In contrast to our approach the weighting is only applied the source model, giving an alternate solution to the target classifier. Adapted to a linear bandit model, the optimization problem can be formulated as:

\begin{equation}
    \hat{\thetabf}=\argmin_\thetabf \norm{\D\thetabf - \y}^2 + \lambda\norm{\thetabf-\thetabf_{src}}^2.
    \label{biased_reg_opt}
\end{equation}

$\thetabf_{src}$ is a convex combination of an arbitrary amount of given source models $\{\thetabf_j\}_{j\in\{1,...,M\}}$:

\begin{equation}
    \thetabf_{src}=\sum_{j=1}^M\alpha_j\thetabf_j,
\end{equation}

As in our model, these weights are not static and are updated after each time step. The update strategy is not chosen to minimize the upper regret bound but can be chosen such that the convex combination is as close as possible to the optimal bandit parameter. The UCB function is then simply given by:

\begin{equation}
    \mathrm{UCB}(a)=\x_a^T\hat{\thetabf}+\gamma\norm{\x_a}_{\A^{-1}(k)},
\end{equation}

 with $\gamma=\sqrt{d\log(1+\frac{k }{d\lambda})+\log(\frac{1}{\delta^2})}+\sqrt{\lambda}\norm{\thetabf_{src}-\thetabf^*}_2$ and the solution to (\ref{biased_reg_opt}):

\begin{equation}
    \hat{\thetabf}=\A^{-1}\D^T\y-(\A^{-1}\D^T\D-\I)\thetabf_{src}.
\end{equation}

At some point in time we expect the weights to converge to a single source bandit closest to the optimal bandit. But contrary to our original model it is not possible for the model to discard  all sources once the target estimation yield better upper bounds for their confidence sets. The upper regret bound is similar to the classic bound with the difference being in one term.

\begin{theorem}
Let \{$\x_{a_k}$\}$_{k=1}^N$ be sequence in $\Rset^d$ and the upper bound of the biggest euclidean distance between any of the $M$ source bandit indexed by $m$ and optimal bandit parameter given by $\max_m\norm{\thetabf_{S,m}-\thetabf^*}\leq U_\mathrm{max}$, then with probability of at least $1-\delta$ the regret of the biased LinUCB algorithm with multiple sources is upper bounded by:
\begin{equation}
R(n)\leq \sqrt{8nd\log(\lambda+n/d)}\left(\sqrt{d\log(1+\frac{n}{d\lambda})+\log(\frac{1}{\delta^2})}+\sqrt{\lambda}U_\mathrm{max}\right)    
\end{equation}

\label{theorem_biased}
\end{theorem}

Since we are looking for an upper bound, $U$ is dominated by the largest euclidean distance between the optimal bandit parameter and all given source bandits. \Cref{theorem_biased} differs from the classic case in the regularization related parameters where we have $\sqrt{\lambda}U_\mathrm{max}$ instead of $\sqrt{\lambda}\norm{\thetabf^*}$. For sources with low values of $U$, we improve the overall regret.

\section{Experimental Results}\label{sec:experiments}
We test the presented algorithms, \ie the weighted model algorithm as well as the biased regularization algorithm, for single source and multiple source transfers on synthetic and real data sets. The plots include the results from the classical LinUCB approach as well as the EXP4 approach from \cite{lattimore2020bandit} with target and source models acting as expert, for comparison purposes. Additionally to the regret plots we also showcase the mean of the target weight as a function of time to see how the relevancy of the target estimation evolved.

\subsection{Synthetic Data Experiments}
Our synthetic experiments follow a similar approach to \cite{liu2018transferable}. The target context feature vectors $\x_a$ are drawn from a multivariate Gaussian with variances sampled from a uniform distribution. We chose the number of dimensions $d=20$ and the number of arms to be 1000. Our optimal target bandit parameter is sampled from a uniform distribution and scaled such that $\norm{\thetabf^*}\leq1$, thus the rewards are implicitly initialized as well with $r =\x_a^T \thetabf^{*}+\epsilon$, with some Gaussian noise $\epsilon\sim\mathcal{N}(0, \sigma^2)$ and $\sigma=1/\sqrt{2\pi}$. The source bandit parameters  $\thetabf_{S,m}$ are initialized by adding a random noise vector $\etabf_m$ to the optimal target bandit parameters for every source bandit to be generated $\thetabf_{S,m} = \thetabf^{*} + \etabf_m$. This way we ensure that there is actual information of the target domain in the source bandit parameter. We could also scale $\etabf_m$ to determine how much information the respective source yields about the target domain. The regularization parameter was constantly chosen to be $\lambda=1$ and the initial weights are equally distributed among all available bandit parameters: $\alpha_T=\alpha_{S,m}=\frac{1}{M+1}$. The shown results are the averaged values over 20 runs.
\begin{figure}
\centering
    \begin{subfigure}[t]{0.49\textwidth}
    \centering
         \includegraphics[width=\textwidth]{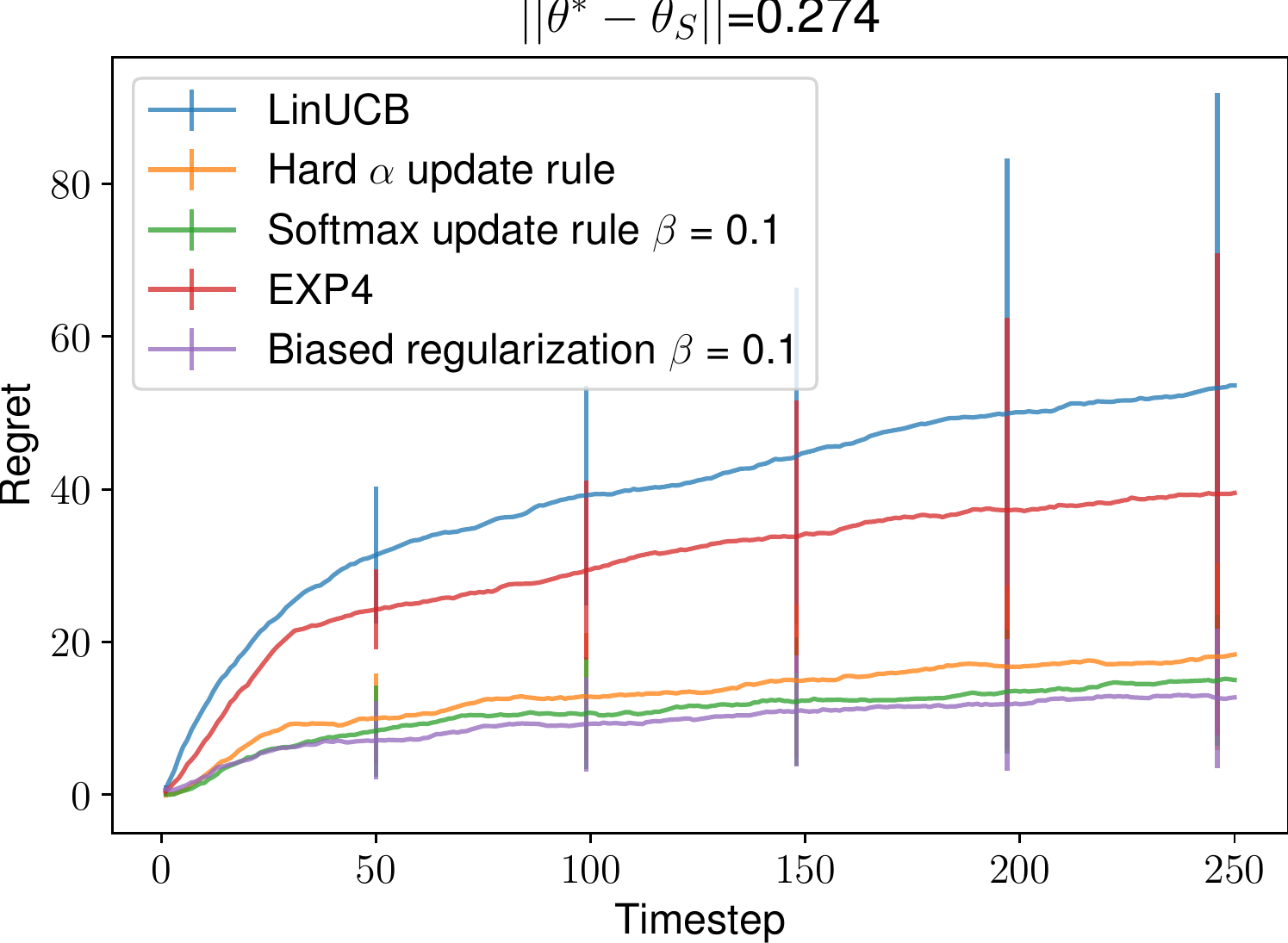}
         \caption{Regret evolution plot labeled by confidence set bound.}
         \label{regret_evo_synth_sing_source}
    \end{subfigure}
    \hfill
        \begin{subfigure}[t]{0.49\textwidth}
    \centering
         \includegraphics[width=\textwidth]{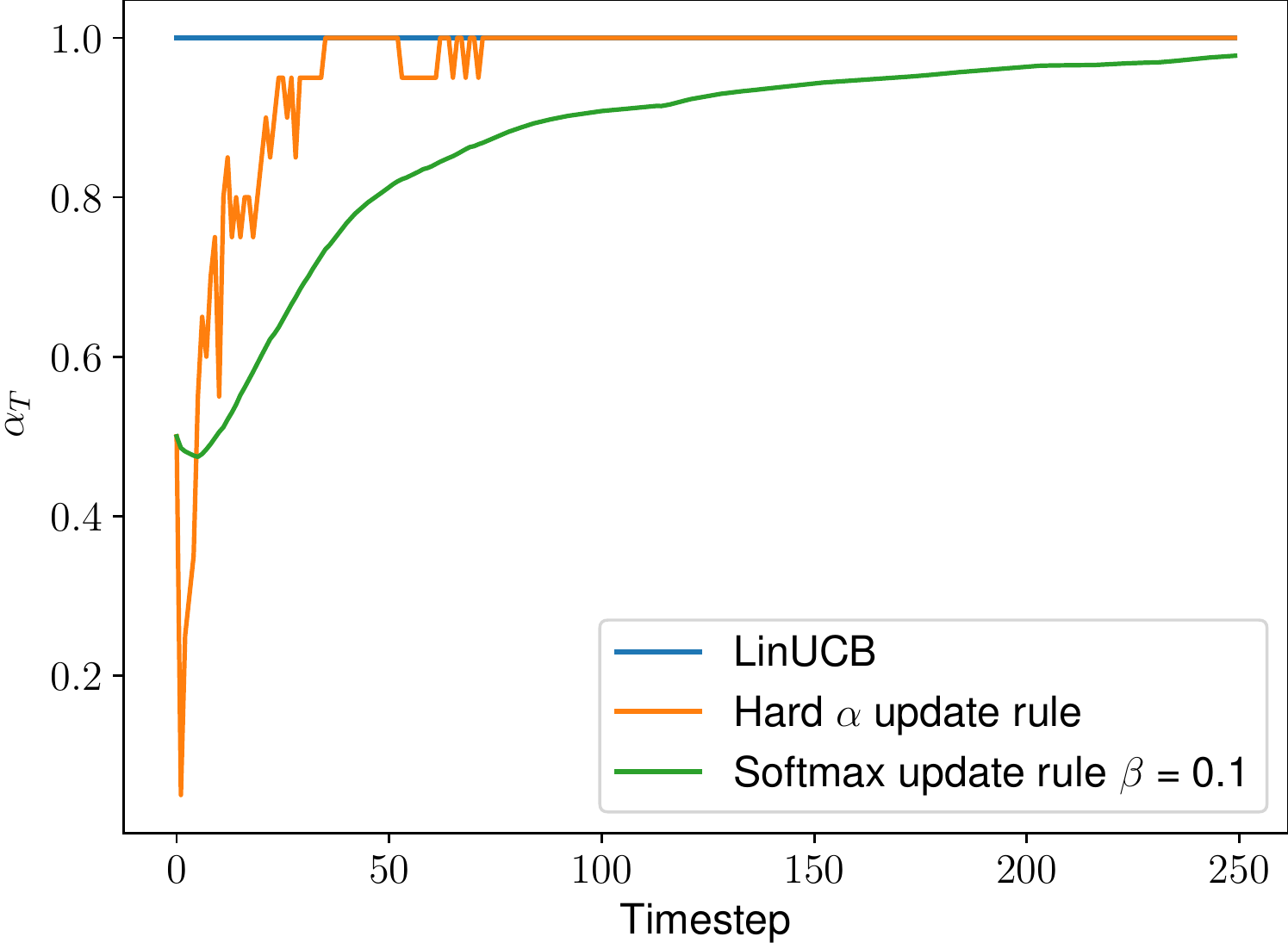}
         \caption{Evolution of the target weight $\alpha_T.$}
    \end{subfigure}
    \label{single_source_plots}
    \caption{Regret and weight evolution for single source transfer scenario on synthetic data sets. The blue lines showcase the classic LinUCB results. The vertical lines indicate the standard deviation.}
\end{figure}

\begin{figure}
\centering
    \begin{subfigure}[t]{0.49\textwidth}
    \centering
         \includegraphics[width=\textwidth]{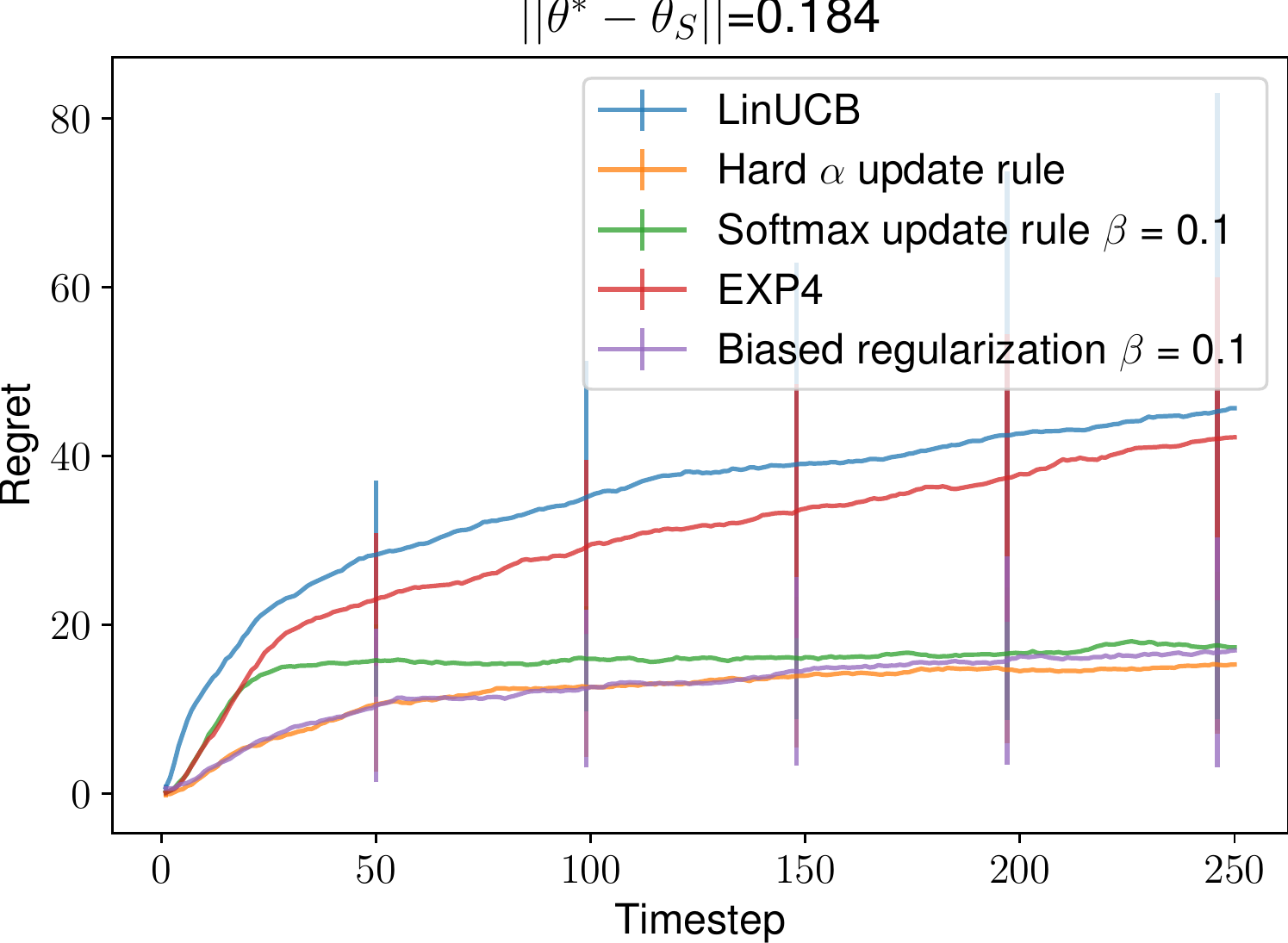}
         \caption{Regret evolution plot labeled by the lowest confidence set bound of all available sources.}
         \label{regret_evo_synth_mult_source}
    \end{subfigure}
    \hfill
        \begin{subfigure}[t]{0.49\textwidth}
    \centering
         \includegraphics[width=\textwidth]{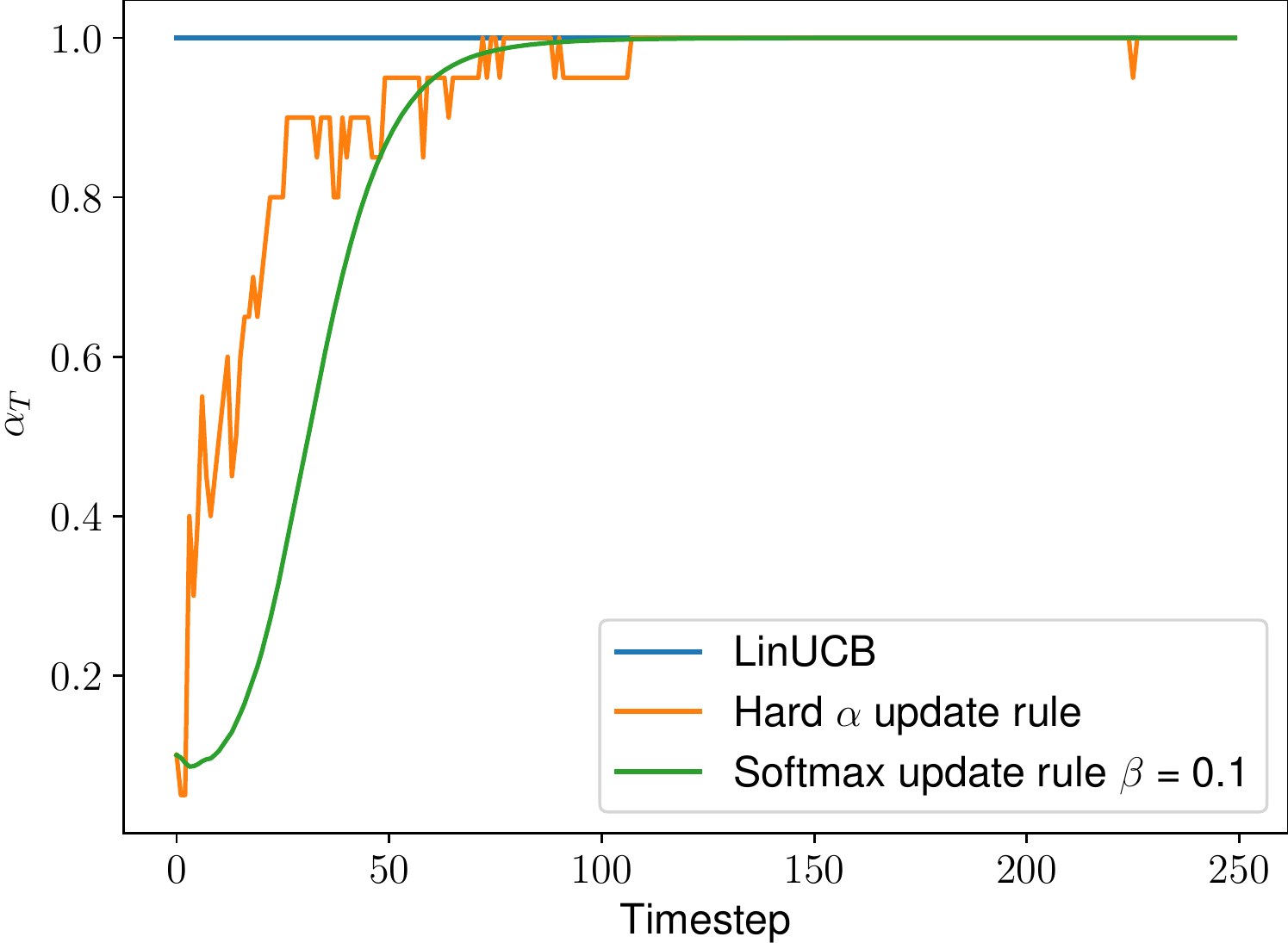}
        \caption{Evolution of the target weight $\alpha_T$. Since multiple sources are present, the initial weight is reduced}
    \end{subfigure}
    \caption{Regret and weight evolution for multiple source transfer scenario (9 sources) on synthetic data sets. The blue lines showcase the classic LinUCB results. The vertical lines indicate the standard deviation.}
    \label{9_sources_plots}
\end{figure}

As we showed in \Cref{sec:weightedLinBandits} the upper regret bound is lower for $\beta \xrightarrow[]{}\infty$ \ie the hard update rule which ignores the KL-divergence in the optimization, but we see overall better results than in the classic case with the softmax update strategy as well. The inclusion of eight more source bandits in Figure \ref{9_sources_plots} improves the sources slightly, though it should be mentioned that all sources generated were similar in quality. Thus we would expect higher improvements in the regret when including significantly better sources. The EXP4 algorithm on the other hand does not perform as well when increasing the number of experts.

\subsection{Real Data Experiments}
The real data sets used for our purposes are taken from the MovieLens sets. Their data include an assemble of thousands of users and corresponding traits such as age, gender and profession as well as thousands of movies and their genres. Every user has a rating from 1 to 5 given to at least 20 different movies. The movies, rated by a user, function as the available arms for that particular user. The information of the movies apart from the title itself are solely given by their genres. Each movie may have up to three different genres and there are 18 different genres in total. Arms, which are linked to the movies, have context vectors depending on the movies genre only. We design 18-dimensional context vector with each dimension representing a genre. If the movie is associated with a particular genre, the respective dimensional feature is set as $x_i=\frac{1}{\sqrt{S}}$ with $S$ as the total number of genres the movie is associated with. This way we guarantee that every context vector is bounded by 1. the reward of an arm in our bandit setting is simply given by the user rating. \\
For our purposes we require source bandits for the transfer learning to take place. Therefore we pretrained a bandit for every single user, given all of the movie information, with the classic LinUCB algorithm and stored the respective parameters. This way every single user can function as a potential source for a different user. With all of the users available we grouped them according their age, gender and profession. We enforce every user to only act as source to other users with similar traits. This stems from a general assumption that people with matching traits may also have similar interests. This is a very general assumption made but given all of the information, it is the easiest way to find likely useful sources for every user. In Figure \ref{real_sources_plots} the results for two individuals of two different groups of users respectively are showcased. Instead of only using one source, we used the multiple source strategy and made use of every user of the same group the individuals are located in, since this way we have a higher chance to find good sources. Even though the real data is far from guaranteed to have a linear reward structure, as well as the fact that important information on the arms' contexts are not available, since ratings usually not only depend on the movie genre, we find satisfying results with converging regrets as well as improved learning rates when including sources.

\begin{figure}
\centering
    \begin{subfigure}[t]{0.49\linewidth}
    \centering
         \includegraphics[width=\textwidth]{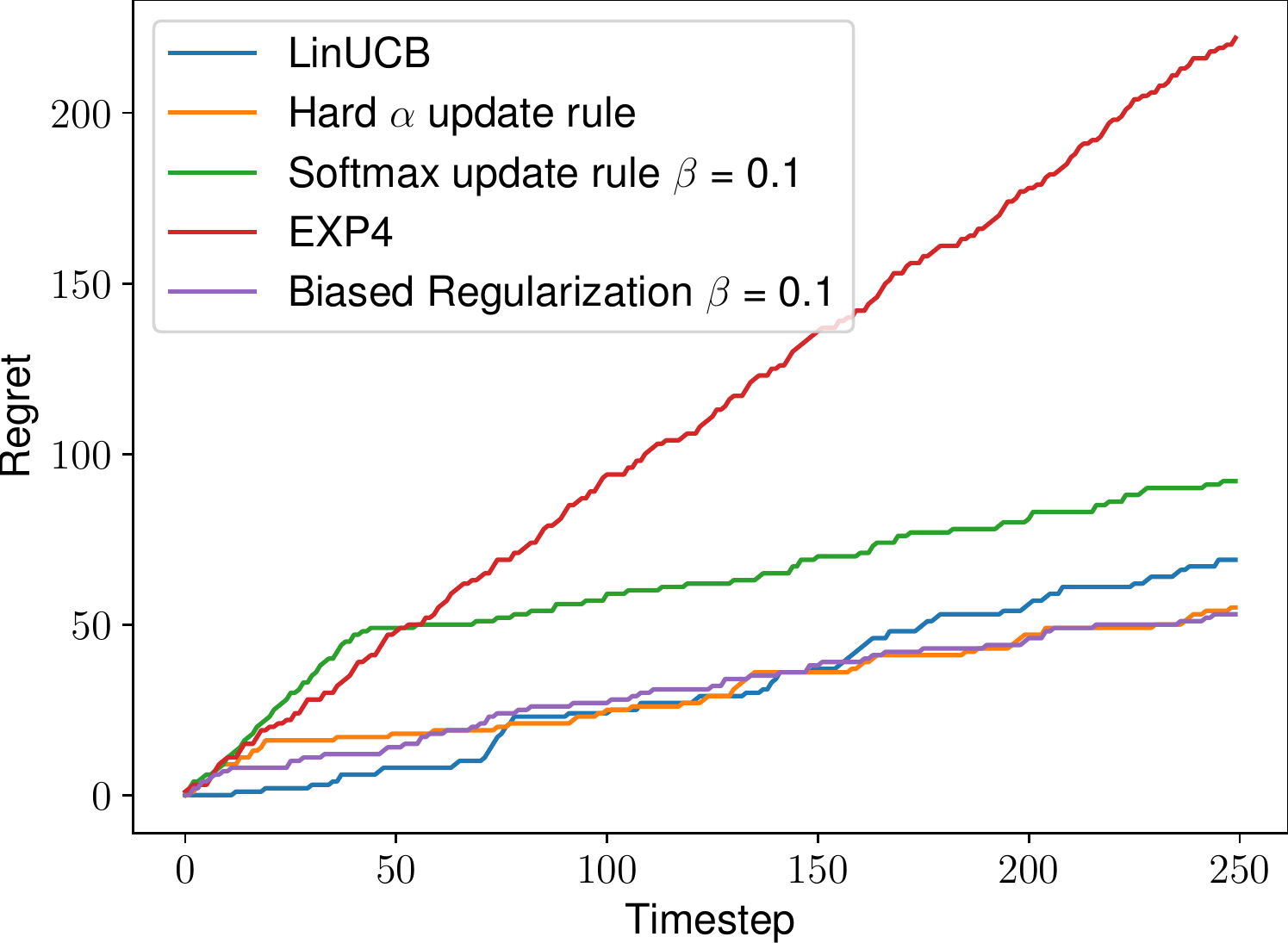}
         \caption{Regret evolution plot with user data taken from the group of 35 to 44 years old female lawyers.}
         \label{regret_evo_real_mult_F_18_0}
    \end{subfigure}
    \hfill
        \begin{subfigure}[t]{0.49\linewidth}
    \centering
         \includegraphics[width=\textwidth]{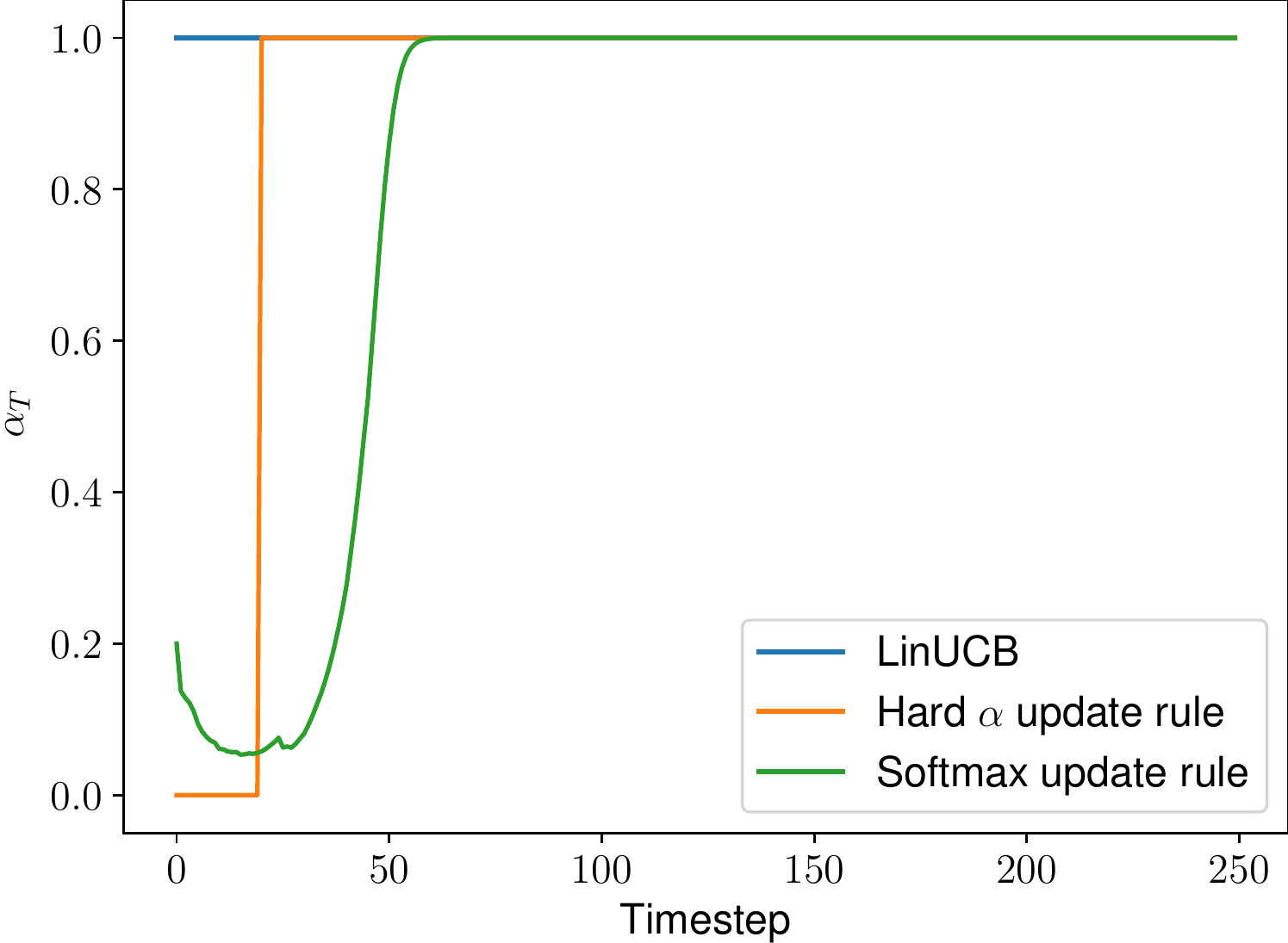}
         \caption{Target weight evolution plot with the respective algorithms labeled with user data taken from the group of 35 to 44 years old female lawyer.}
         \label{regret_evo_real_mult_M_35_3}
    \end{subfigure}
    \caption{Regret evolution for multiple source transfer scenario on real data sets taken from Movielens data. A group of users are shown with one bandit trained for a random user of each group, while the rest of the users act as source to the respective user. The blue lines showcase the classic LinUCB results.}
    \label{real_sources_plots}
\end{figure}

\section{Discussion and Outlook}
This work shows that our approach to make use of information from different tasks, without having actually access to concrete data points, is efficient, given the improved regrets. We have proven an upper regret bound of our weighted LinUCB algorithm with the hard update strategy at least as good as the classic LinUCB bound with a regret rate of $O(d\sqrt{n\log{n}})$, and a converging sub-linear negative-transfer term when using the softmax update strategy. Further argument for the utility of our model was given with synthetic and real data experiments. The synthetic data sets showed promising results especially with the softmax update strategy, even without having a guaranteed improved regret bound. The softmax approach uses a convex combination of models, which might be more practical than using one model at a time especially when it comes to high quality sources. This further raises the question whether different weighting update rules, which yield solutions consisting of a span of source models, might be more efficient for transfer. The inclusion of multiple sources further improved the results, indicating that using information from multiple different tasks is more effective then just one, which aligns with our theoretical result in Theorem \ref{theorem_mult_hard}. The real-world data experiments showed improvements as well, even when considering that the rewards did not necessarily follow a linear model and that the available features for the context vector were rather sparse, the transfer of information from similar users almost always led to lower regrets.

In upcoming projects we intend to adapt our approach to non-linear models such as kernelized bandits, since the convex weighting is not limited to just linear models, as well as give a proper regret bound for the softmax update strategy. There is potential in using our transfer model to non stationary bandits, such that each prior estimation of the bandit parameter may act as source for the current setting, thus making use of the information collected in prior instances of the bandit setting. In this case we would need to make assumptions of the change rate of the tasks after a certain amount of time steps. Previous algorithms on non-stationary bandits \cite{russac2019weighted} perform weighting on data points and discard them after some time steps, without evaluating the benefit of the data  beforehand. In our setting, previously trained bandit parameters would be used according to their performance.

\subsubsection{Acknowledgements}
This work was supported by Grant 01IS20051 from the German Federal Ministry of Education and Research (BMBF). S. Maghsudi is a member of the Machine Learning Cluster of Excellence, EXC number 2064/1 – Project number 390727645. The authors thank the International Max Planck Research School for Intelligent Systems (IMPRS-IS)
for supporting Steven Bilaj.

\newpage
\appendix
\section{Proof of Theorem 1}
We use three lemmas from \cite{abbasi2011improved} used in their regret analysis, as we use the results for our proof.
\begin{lemma} (Self-normalized bound for vector-valued martingales). Let $\tau$ be a stopping time with respect to the filtration $\{F_t\}_{t=0}^\infty$ and define $\S_t=\D^T(\tau)\epsilonbf$, with $\epsilonbf$ as a subgaussian noise vector. Then, for $\delta > 0$, with probability $1-\delta$,

\begin{equation}
    \norm{\S_{\tau}}^2_{\A^{-1}(\tau)}\leq \log(\frac{\det(\A(\tau))}{\delta^2\lambda^d})
\end{equation}
\label{Lemma_vector_martingale}
\end{lemma}

\begin{lemma}
Suppose $\x_{a_1},\x_{a_2},...,\x_{a_n}\in\Rset^d$ and for any $1\leq k\leq n$, $\norm{\x_{a_k}}\leq 1$. Let $\A=\lambda \I +\sum_{k=1}^n\x_{a_k}\x_{a_k}^T$ for some $\lambda > 0$. Then,

\begin{equation}
    \det(\A) \leq (\lambda+n/d)^d
\end{equation}
\end{lemma}

\begin{lemma}
(Confidence Set Bound). Suppose $\x_{a_1},\x_{a_2},...,\x_{a_n}\in\Rset^d$ and for any $1\leq k\leq n$, $\norm{\x_{a_k}}\leq 1$. Let $\A=\lambda \I +\sum_{k=1}^n\x_{a_k}\x_{a_k}^T$ for some $\lambda > 0$ and assume $\norm{\thetabf^*}\leq 1$, with $\hat{\thetabf}_T=\A^{-1}\D^T\y$. Then, for any $\delta > 0$, with probability of at least 1-$\delta$ we have:

\begin{equation}
    \norm{\hat{\thetabf}_T-\thetabf^*}_{\A}\leq \sqrt{d\log(\frac{(\lambda+n/d)}{\lambda})+\log(\frac{1}{\delta^2})}+\sqrt{\lambda} 
\end{equation}
\label{Lemma_confidence_set}
\end{lemma}

\begin{lemma}
Suppose $\x_{a_1},\x_{a_2},...,\x_{a_n}\in\Rset^d$ is a sequence and for any $1\leq k\leq \infty$, $\norm{\x_k}\leq 1$. Let $\A(n)=\lambda \I +\sum_{k=1}^n\x_k\x_k^T$ for some $\lambda > 0$ then

\begin{equation}
    \sum_{k=1}^n\norm{\x_{a_k}}^2_{A^{-1}(k)} \leq 2d\log(1+\frac{k}{d\lambda})
\end{equation}
\label{Lemma_sum_x}
\end{lemma}

\begin{proof}[Proof of Theorem 1]
First, similar to the confidence set bound in \Cref{Lemma_confidence_set}, we require a confidence set bound for $\thetabf_S$ in the sense of the $\norm{.}_{\A(k)}$ norm. Since it remains constant, determining its confidence set bound is straightforward. Assuming $\norm{\thetabf_S-\thetabf^*}_2= U$, we have

\begin{align*}
    \norm{\thetabf_S-\thetabf^*}_{\A(k)}&=\sqrt{\lambda \norm{\thetabf_S-\thetabf^*}^2_2+\sum_{i=1}^k(\x_{a_k}^T(\thetabf_S-\thetabf^*))^2}\\
    &\leq\sqrt{\lambda \norm{\thetabf_S-\thetabf^*}^2_2 +\norm{\thetabf_S-\thetabf^*}^2_2\sum_{i=1}^k\norm{x_{a_k}}_2^2}\\
    &\leq U \sqrt{\lambda+k},
\end{align*}

where in the last inequality, the Cauchy-Schwarz inequality was used and the fact that $\norm{\x_{a_k}}\leq 1$. For simplicity we define the upper confidence set bounds for $\thetabf_S$ as $\gamma_S$ and for $\hat{\thetabf}_T$ as $\gamma_T$. Now we determine the maximal time step $\kappa$ at which $\gamma_S \leq \gamma_T$ is guaranteed. For this we require a lower bound for $\gamma_T$:

\begin{equation}
    \gamma_T = \sqrt{d\log(1+\frac{k}{\lambda d})+\log(\frac{1}{\delta^2})}+\sqrt{\lambda} \geq \sqrt{\frac{2k/\lambda}{2+k/(\lambda d)}+\log(\frac{1}{\delta^2})},
\end{equation}

where we used $\log(1+x)\geq \frac{2x}{2+x}$ for $x\in(0,\infty)$. With this lower bound we can analytically determine a lower bound $\kappa$:

\begin{align*}
     &U \sqrt{\lambda+\kappa} = \sqrt{\frac{2\kappa/\lambda}{2+\kappa/(\lambda d)}+\log(\frac{1}{\delta^2})} \\
    &\iff U^2(\lambda + \kappa) - \frac{2k}{2\lambda + \kappa/d} - \log(\frac{1}{\delta^2}) = 0 \\
    &\iff \kappa^2 + \left(2\lambda d + \lambda - \frac{\log(\frac{1}{\delta^2})+2d}{U^2}\right)\kappa + 2\lambda d \left(\lambda -\frac{\log(\frac{1}{\delta^2})}{U^2}\right) = 0 
\end{align*}

Which is a quadratic inequality with respect to $\kappa$ and yields the following solution, with the condition $\delta\leq\exp(-2\lambda)$ we retrieve yet another lower bound: 

\begin{align*}
    \kappa &= \sqrt{\left(\frac{d+\log(\frac{1}{\delta})}{U^2}\right)^2+\left[\lambda\left(d-\frac{1}{2}\right)\right]^2+2\lambda\frac{\log(\frac{1}{\delta})(d-\frac{1}{2})-d(d+\frac{1}{2})}{U^2}} \\
    &-\lambda(d+\frac{1}{2})+\frac{d+\log(\frac{1}{\delta})}{U^2}\\
    &\geq \sqrt{\left(\frac{d+\log(\frac{1}{\delta})}{U^2}\right)^2+\left[\lambda\left(d+\frac{1}{2}\right)\right]^2-2\lambda\frac{\log(\frac{1}{\delta})(d+\frac{1}{2})+d(d+\frac{1}{2})}{U^2}} \\
    &-\lambda(d+\frac{1}{2})+\frac{d+\log(\frac{1}{\delta})}{U^2}\\
    &= \sqrt{\left(\frac{d+\log(\frac{1}{\delta})}{U^2}-\lambda\left(d+\frac{1}{2}\right)\right)^2}-\lambda(d+\frac{1}{2})+\frac{d+\log(\frac{1}{\delta})}{U^2}\\
    &= 2\left[d\left(\frac{1}{U^2}-\lambda\right)+\frac{\log(\frac{1}{\delta})}{U^2}-\frac{\lambda}{2}\right]\\
    &\geq\left\lfloor2\left[d\left(\frac{1}{U^2}-\lambda\right)+\lambda\left(\frac{2}{U^2}-\frac{1}{2}\right)\right]\right\rfloor
\end{align*}

Where in the first inequality we used
\begin{align*}
    \lambda^2\left(d+\frac{1}{2}\right)^2-\lambda^2\left(d-\frac{1}{2}\right)^2&\leq2\lambda\frac{\log(\frac{1}{\delta})(d-\frac{1}{2})}{U^2}-\left(-2\lambda\frac{\log(\frac{1}{\delta})(d+\frac{1}{2})}{U^2}\right)\\
\iff    2\lambda^2d&\leq\frac{4\lambda\log(\frac{1}{\delta})d}{U^2},
\end{align*}

which holds when $\delta\leq\exp(-\lambda/2)$.
Next we we give the upper confidence bound of our model:

\begin{equation}
|\x_{a_k}^T\hat{\thetabf} - \x_{a_k}^T\thetabf^*|\leq \norm{\hat{\thetabf}-\thetabf^*}_{\A(k)} ||\x_{a_k}||_{\A^{-1}(k)}=\Delta(k),
\end{equation}

where the Cauchy-Schwarz inequality was used. We denote the exploration term of the UCB as $\Delta(k)$.
We construct the regret as a sum of immediate regrets $\rho(k)$:

\begin{equation}
    R(n)=\sum_{k=1}^n \rho(k),
\end{equation}

as for the immediate regrets, we define the context vector yielding the highest reward $\x_{a^*}$. We then have

\begin{align*}
    \rho(k) &=\x_{a^*}^T\thetabf^* - \x_{a_k}^T\thetabf^*\\
   &\leq \x_{a_k}^T\hat\thetabf + \Delta(k) - \x_{a_k}^T\thetabf^* \\
   &\leq \x_{a_k}^T\hat\thetabf + \Delta(k) - \x_{a_k}^T\hat\thetabf + \Delta(k) \\
   &=2\Delta(k).
\end{align*}

The first inequality makes use of the UCB principle optimism in the face of uncertainty and the second inequality results from the definition of the confidence set used for the exploration term. The resulting total regret can then be bounded:

\begin{align*}
    R(n)\; &= \sum_{k=1}^n \rho(k) \\
    &\leq2\sum_{k=1}^n\Delta(k)=2\sum_{k=1}^n\norm{\hat{\thetabf}-\thetabf^*}_{\A(k)}\norm{\x_{a_k}}_{\A^{-1}(k)} \\
    &\leq 2\sum_{k=1}^n\left[\alpha_{S}(k)||\thetabf_S^*-\thetabf^*||_{\A(k)} + \alpha_{T}(k)||\hat\thetabf_{T}(k)-\thetabf^*||_{\A(k)}\right]||\x_{a_k}||_{\A^{-1}(k)} \\
    &\leq 2\sum_{k=1}^n \alpha_{S}(k)\left(U\sqrt{\lambda+k} - ||\hat\thetabf_{T}(k)-\thetabf^*||_{\A(k)} \right)||\x_{a_k}||_{\A^{-1}(k)} \\
    &+2\sum_{k=1}^n||\thetabf_{T}(k)-\thetabf^*||_{\A(k)}||\x_{a_k}||_{\A^{-1}(k)} \\
    & \leq 2U \sqrt{\lambda + \kappa}\sqrt{\kappa\sum_{k=1}^\kappa||\x_{a_k}||^2_{\A^{-1}(k)}}-R_T(\kappa) + R_T(n) \\
    &\leq U\sqrt{8\kappa(\lambda+\kappa)d\log(1+\frac{\kappa}{d\lambda})}-R_T(\kappa) + R_T(n)
\end{align*}

While we used Lemma \ref{Lemma_confidence_set} in the fourth inequality, resulting into the classic regret and Lemma \ref{Lemma_sum_x} in the last step.
\qed
\end{proof}

\section{Proof of Theorem 2}

\begin{proof}[Proof of Theorem 2]
We assume that $\gamma_S > \gamma_T$ from the very beginning:

\begin{equation}
    R(n) \leq \sum_{k=1}^n\alpha_{S}(k)(\gamma_{S}(k)-\gamma_{T}(k))\norm{\x_{a_k}}_{\A^{-1}(k)} + R_T(n)
\end{equation}

with $R_T(n)$ as the traditional regret bound for LinUCB. We define $\Delta_{\gamma(k)}=\gamma_{S}(k) - \gamma_{T}(k)$. We are taking a closer look at the worst case scenario with $\Delta_{\gamma(k)}>0$ First we show how the weights evolve in the softmax approach:

\begin{equation}
    \alpha_{S}(k)=\frac{1}{1 + Z(k)(\frac{1}{\alpha_{S}(k-1)}-1)}=\frac{1}{1+\prod_{i=1}^{k} Z(i)(\frac{1}{\alpha_{S}(0)}-1)},
\end{equation}

with $Z(k)=\exp(\beta\Delta_{\gamma(k)})$, thus in case $\Delta_{\gamma(k)}>0$ for all $k$ we can further bound the regret as:

\begin{equation}
    R(n)\leq\sum_{k=1}^n\frac{1}{1+\prod_{i=1}^{k} Z(i)(\frac{1}{\alpha_{S}(0)}-1)}\Delta_{\gamma(k)}\norm{\x_{a_k}}_{\A^{-1}(k)}+R_T(n)
\end{equation}

looking at the first sum we know that for large values of $\Delta_{\gamma(k)}$ the sigmoid term decreases rapidly by taking the upper bound: $\frac{1}{1+\prod_{i=1}^k Z(i)(\frac{1}{\alpha_{S}(0)}-1)}\leq\frac{1}{\prod_{i=1}^k Z(i)(\frac{1}{\alpha_{S}(0)}-1)}$ we can minimize this locally by setting $\Delta_{\gamma(k)}=\frac{1}{\beta}$ for the kth summand respectively. thus we can further estimate our upper regret bound such that

\begin{equation}
    R(n) \leq \sum_{k=1}^n\frac{\exp(-\beta\sum_{i=1}^{k-1}\Delta_{\gamma_i})}{e\beta(\frac{1}{\alpha_{S}(0)}-1)}\norm{\x}_{\A^{-1}(k)} + R_T(n),
\end{equation}

From here we will focus on the negative transfer term only. Since we know that $\Delta_{\gamma(i)}$ grows with each time step, as well as in this case it is supposed to be positive for a bad source scenario, we can further estimate:

\begin{equation}
     \sum_{k=1}^n\frac{\exp(-\beta\sum_{i=1}^{k-1}\Delta_{\gamma_i})}{e\beta(\frac{1}{\alpha_{S}(0)}-1)}\norm{\x}_{\A^{-1}(k)} \leq \sum_{k=1}^n\frac{\exp(-\beta (k-1)\Delta_{\mathrm{min}})}{e\beta(\frac{1}{\alpha_{S}(0)}-1)}
\end{equation}

where we used $\norm{x}_{A^{-1}(k)}\leq 1$ and defined $\Delta_\mathrm{min}=\min_k\Delta_{\gamma(k)}$. With the use of the geometric series we finally obtain:

\begin{equation}
    \sum_{k=0}^{n-1}\frac{\exp(-\beta k\Delta_{\mathrm{min}})}{e\beta(\frac{1}{\alpha_{S}(0)}-1)} \leq \frac{(1-\alpha_{T}(0))}{e\beta\alpha_{T}(0)(1-\exp(-\beta\Delta_{\mathrm{min}}))},
    \label{neg_transf_last_step}
\end{equation}

where we changed the sum indices in (\ref{neg_transf_last_step}) and applied the geometric series formula.
\qed
\end{proof}

\section{Proof of Theorem 3}
\begin{proof}[Proof of Theorem 3]
The proof is analogous to theorem 1 with the difference that multiple sources are available. Due to the algorithm it always picks the source with the lowest confidence set bound, denoted by: $\norm{\thetabf_{S,m}-\thetabf^*}_{A(k)}\leq \min_m U_m\sqrt{\lambda+k}=U_\mathrm{min}\sqrt{\lambda+k}$. Using this, the rest of the proof follows the same steps as Theorem 1.
\qed
\end{proof}
\section{Proof of Theorem 4}
\begin{proof}[Proof of Theorem 4]
We assume $\gamma_{S,j}(k) > \gamma_T(k)$ for all $j\in{1,...,M}$ from the very beginning:

\begin{equation}
    R_n \leq \sum_{k=1}^n\sum_j^M\alpha_{S,j}(k)(\gamma_{S,j}(k)-\gamma_{T}(k))\norm{\x_{a_k}}_{\A^{-1}(k)} + R_T
\end{equation}

with $R_T$ as the traditional regret bound for LinUCB. We define $\Delta_{j}(j)=\gamma_{S,j}(k) - \gamma_{T}(k)$. We are taking a closer look at the worst case scenario with $\Delta_{j}(k)>0$ for all $j$. First we show how the weights evolve in the softmax approach:

\begin{align*}
    \alpha_{S, j}(k)&=\frac{1}{1 +\sum_{i\neq j}\frac{\alpha_{S,i}(k-1)}{\alpha_{S,j}(k-1)}\exp(\beta(\gamma_{S,j}(k)-\gamma_{S,i}(k))) + \exp(\beta\Delta_j(k))\frac{\alpha_T(k-1)}{\alpha_{S,j}(k-1)}}\\
    &=\frac{1}{1+\sum_{i\neq j}\exp(\beta\sum_{l=1}^k(\gamma_{S,j}(l)-\gamma_{S,i}(l))) + M\exp(\beta\sum_{l=1}^k\Delta_j(l))(\frac{\alpha_T(0)}{1-\alpha_T(0)})}\\
    &\leq \frac{1}{M\exp(\beta\sum_{l=1}^k\Delta_j(l))(\frac{\alpha_T(0)}{1-\alpha_T(0)})},
\end{align*}

were we assumed that each inital source weight is set to $\alpha_{S,j}(0)=\frac{1-\alpha_T(0)}{M}$, thus in case $\Delta_{j}(k)>0$ for all $j$ and $k$ we can further bound the regret as:

\begin{equation}
    R_n\leq\sum_{k=1}^n\sum_{j=1}^M\frac{1}{M\exp(\beta\sum_{l=1}^k\Delta_j(l))(\frac{\alpha_T(0)}{1-\alpha_T(0)})}\Delta_{j}(k)\norm{\x_{a_k}}_{\A^{-1}(k)}+R_T
\end{equation}

we know that for large values of $\Delta_{j}$ the respective term decreases rapidly we can minimize these locally by setting $\Delta_{j}(k)=\frac{1}{\beta}$ for the kth summand respectively for all sources. Thus we can further estimate the negative transfer term such that

\begin{equation}
    R_n \leq \sum_{k=1}^n\sum_{j=1}^M\frac{\exp(-\beta\sum_{l=1}^{k-1}\Delta_{j}(l))}{eM\beta\frac{\alpha_T(0)}{1-\alpha_T(0)}}\norm{\x_{a_k}}_{\A^{-1}(k)} + R_T.
\end{equation}

From here we will ignore the classic regret term $R_T$ and use it again at the end. Since we know that $\Delta_{j}$ grows with each time step, as well as in this case it is supposed to be positive for a bad source scenario, we can further estimate the negative transfer term:

\begin{equation*}
    \sum_{k=1}^n\sum_{j=1}^M\frac{\exp(-\beta\sum_{l=1}^{k-1}\Delta_{j}(l))}{eM\beta\frac{\alpha_T(0)}{1-\alpha_T(0)}}\norm{\x_{a_k}}_{\A^{-1}(k)} \leq \sum_{k=1}^n\sum_{j=1}^M\frac{\exp(-\beta (k-1)\Delta_{\mathrm{min},j})}{eM\beta\frac{\alpha_T(0)}{1-\alpha_T(0)}}
\end{equation*}

where we used $\norm{x}_{A^{-1}(k)}\leq 1$ and used $\Delta_{\mathrm{min},j}=\min_k\Delta_{j}(k)$. With the use of the geometric series we finally obtain:

\begin{equation}
    \sum_{k=0}^{n-1}\sum_{j=1}^M\frac{\exp(-\beta k\Delta_{\mathrm{min},j})}{eM\beta\frac{\alpha_T(0)}{1-\alpha_T(0)}} \leq\sum_{j=1}^M \frac{(1-\alpha_{T}(0))}{eM\beta\alpha_{T}(0)(1-\exp(-\beta\Delta_{\mathrm{min},j}))},
    \label{neg_mult_transf_last_step}
\end{equation}

where we changed the sum indices in (\ref{neg_mult_transf_last_step}) and applied the geometric series formula.
\qed
\end{proof}

\section{Proof of Theorem 5}
The proof of the next Lemma and Theorem is adapted from \cite{abbasi2011improved}.
\begin{lemma}
Suppose $\x_{a_1},\x_{a_2},...,\x_{a_n}\in\Rset^d$ and for any $1\leq k\leq n$, $\norm{\x_{a_k}}\leq 1$. Let $D=\{x_{a_i}\}_{i=1}^{k-1}$, $\A=\lambda \I +\sum_{k=1}^n\x_{a_k}\x_{a_k}^T$ for some $\lambda > 0$ and assume $\norm{\thetabf^*}\leq 1$. A source bandit parameter $\thetabf_S$ is given as well. With the estimation $\hat{\thetabf}_T=\A^{-1}\D^T\y - (\A^{-1}\D^T\D - \I)\thetabf_S$, then, for any $\delta > 0$, with probability of at least 1-$\delta$ we have:
\begin{equation}
    \norm{\hat{\thetabf}-\thetabf^*}_A \leq \sqrt{d\log(1+\frac{k}{d\lambda})-2\log(\delta)}+\sqrt{\lambda}\norm{\thetabf_S-\thetabf^*}_2
\end{equation}

\begin{proof}
 Suppose $\x_{a_1},\x_{a_2},...,\x_{a_n}\in\Rset^d$ and for any $1\leq k\leq n$, $\norm{\x_{a_k}}\leq 1$. Let $D=\{x_{a_i}\}_{i=1}^{k-1}$, $\A=\lambda \I +\sum_{k=1}^n\x_{a_k}\x_{a_k}^T$ for some $\lambda > 0$ and assume $\norm{\thetabf^*}\leq 1$. A source bandit parameter $\thetabf_S$ is given as well. With the estimation $\hat{\thetabf}_T=\A^{-1}\D^T\y - (\A^{-1}\D^T\D - \I)\thetabf_S$, then, for any $\delta > 0$, with probability of at least 1-$\delta$ we have:
 
\begin{align*}
    \hat{\thetabf}&=\A^{-1}\D^T\y - (\A^{-1}\D^T\D - \I)\thetabf_S \\
    &=\A^{-1}\D^T(\D\thetabf^*+\epsilonbf)-\A^{-1}\D^T\D\thetabf_S+\thetabf_S \\
    &=\thetabf^*-\lambda\A^{-1}\thetabf^*+\A^{-1}\D^T\epsilonbf+\lambda\A^{-1}\thetabf_S
\end{align*}

Next by shifting $\thetabf^*$ to the left as well as applying the Cauchy-Schwarz inequality after doing using a scalar product with $\x$ we get:

\begin{equation}
\langle\hat{\thetabf}-\thetabf^*,x\rangle\leq\norm{x}_{\A^{-1}}(\norm{\D^T\epsilonbf}_{\A^{-1}}+\lambda\norm{\thetabf_S-\thetabf^*}_{\A^{-1}}),
\end{equation}

next by using $\norm{\thetabf_S-\thetabf^*}^2_{A^{-1}}\leq 1/\lambda\norm{\thetabf_S-\thetabf^*}^2_2$, Lemma \ref{Lemma_vector_martingale} and by plugging in $x=A(\hat{\thetabf}-\thetabf^*)$ we get:

\begin{equation*}
    \norm{\hat{\thetabf}-\thetabf^*}^2_A \leq \norm{\hat{\thetabf}-\thetabf^*}_A\left(\sqrt{d\log(1+\frac{k}{d\lambda})+\log(\frac{1}{\delta^2})}+\sqrt{\lambda}\norm{\thetabf_S-\thetabf^*}_2\right)
\end{equation*}

thus as confidence set required for our UCB we get:

\begin{equation}
    \norm{\hat{\thetabf}-\thetabf^*}_A \leq \sqrt{d\log(1+\frac{k}{d\lambda})+\log(\frac{1}{\delta^2})}+\sqrt{\lambda}\norm{\thetabf_S-\thetabf^*}_2
\end{equation}
\qed
\end{proof}
\label{Lemma_biased_confidence_set}
\end{lemma}

\begin{proof}[Proof of Theorem 5]
We give the upper confidence bound of the biased regularization model:

\begin{equation}
|\x_{a_k}^T\hat{\thetabf} - \x_{a_k}^T\thetabf^*|\leq \norm{\hat{\thetabf}-\thetabf^*}_{\A(k)} ||\x_{a_k}||_{\A(k)^{-1}}=\Delta(k),
\end{equation}

where the Cauchy-Schwarz inequality was used. The next steps, are mostly identical to Theorem 1. We denote the exploration term of the UCB as $\Delta(k)$.
We construct the regret as a sum of immediate regrets $\rho(k)$:

\begin{equation}
    R_n=\sum_{k=1}^n \rho(k),
\end{equation}

as for the immediate regrets, we define the context vector yielding the highest reward $\x_{a^*}$. We then have

\begin{align*}
    \rho(k) &=\x_{a^*}^T\thetabf^* - \x_{a_k}^T\thetabf^*\\
   &\leq \x_{a_k}^T\hat\thetabf + \Delta(k) - \x_{a_k}^T\thetabf^* \\
   &\leq \x_{a_k}^T\hat\thetabf + \Delta(k) - \x_{a_k}^T\hat\thetabf + \Delta(k) \\
   &=2\Delta(k).
\end{align*}

The first inequality makes use of the UCB principle optimism in the face of uncertainty and the second inequality results from the definition of the confidence set used for the exploration term. The resulting total regret can then be bounded:

\begin{align*}
    R_n\; &= \sum_{k=1}^n \rho(k) \\
    &\leq2\sum_{k=1}^n\Delta(k)=2\sum_{k=1}^n\norm{\hat{\thetabf}-\thetabf^*}_{\A(k)}\norm{\x_{a_k}}_{\A(k)^{-1}} \\
    &\leq 2\norm{\hat{\thetabf}-\thetabf^*}_{\A_n}\sqrt{n\sum_{k=1}^n\norm{\x_{a_k}}^2_{A(k)^{-1}}} \\
    &\leq \left(\sqrt{d\log(1+\frac{n}{d\lambda})+\log(\frac{1}{\delta^2})}+\sqrt{\lambda}U\right)\sqrt{8nd\log(1+\frac{n}{d\lambda})}
\end{align*}

While we used Lemma \ref{Lemma_biased_confidence_set} and Lemma \ref{Lemma_sum_x} in the last step as well as $U=\norm{\frac{1}{M}\sum_{j=1}^M U_j}\geq\norm{\thetabf_S-\thetabf^*}_2$.
\qed
\end{proof}

\begin{thebibliography}{8}
\bibitem{abbasi2011improved}Abbasi-Yadkori, Y., Pál, D. \& Szepesvári, C. Improved Algorithms for Linear Stochastic Bandits. {\em Advances In Neural Information Processing Systems}. (2011)
\bibitem{amrallah2020radio}Amrallah, A., Mohamed, E., Tran, G. \& Sakaguchi, K. Radio Resource Management Aided Multi-Armed Bandits for Disaster Surveillance System. {\em Proc. 2020 International Conference On Emerging Technologies For Communications (ICETC2020), Virtual, K1-4}. (2020)
\bibitem{audibert2009exploration}Audibert, J., Munos, R. \& Szepesvári, C. Exploration–exploitation tradeoff using variance estimates in multi-armed bandits. {\em Theoretical Computer Science}. (2009)
\bibitem{auer2002nonstochastic}Auer, P., Cesa-Bianchi, N., Freund, Y. \& Schapire, R. The nonstochastic multiarmed bandit problem. {\em SIAM Journal On Computing}. (2002)
\bibitem{azar2013sequential}Azar, M., Lazaric, A. \& Brunskill, E. Sequential transfer in multi-armed bandit with finite set of models. {\em Advances In Neural Information Processing Systems}. (2013)
\bibitem{bouneffouf2020survey}Bouneffouf, D., Rish, I. \& Aggarwal, C. Survey on applications of multi-armed and contextual bandits. {\em 2020 IEEE Congress On Evolutionary Computation (CEC)}. (2020)
\bibitem{bush1953stochastic}Bush, R. \& Mosteller, F. A stochastic model with applications to learning. {\em The Annals Of Mathematical Statistics}. (1953)
\bibitem{chu2011contextual}Chu, W., Li, L., Reyzin, L. \& Schapire, R. Contextual Bandits with Linear Payoff Functions. {\em AISTATS}. (2011)
\bibitem{du2017hypothesis}Du, S., Koushik, J., Singh, A. \& Póczos, B. Hypothesis transfer learning via transformation functions. {\em Advances In Neural Information Processing Systems}. (2017)
\bibitem{duan2009domain}Duan, L., Tsang, I., Xu, D. \& Chua, T. Domain adaptation from multiple sources via auxiliary classifiers. {\em Proceedings Of The 26th Annual International Conference On Machine Learning}. (2009)
\bibitem{kuzborskij2013stability}Kuzborskij, I. \& Orabona, F. Stability and hypothesis transfer learning. {\em International Conference On Machine Learning}. (2013)
\bibitem{kuzborskij2017fast}Kuzborskij, I. \& Orabona, F. Fast rates by transferring from auxiliary hypotheses. {\em Machine Learning}. (2017)
\bibitem{labille2021transferable}Labille, K., Huang, W. \& Wu, X. Transferable Contextual Bandits with Prior Observations. {\em Pacific-Asia Conference On Knowledge Discovery And Data Mining}. (2021)
\bibitem{langford2007epoch}Langford, J. \& Zhang, T. The epoch-greedy algorithm for multi-armed bandits with side information. {\em Advances In Neural Information Processing Systems}. (2007)
\bibitem{lattimore2020bandit}Lattimore, T. \& Szepesvári, C. Bandit Algorithms.  (2020)
\bibitem{li2010contextual}Li, L., Chu, W., Langford, J. \& Schapire, R. A contextual-bandit approach to personalized news article recommendation. {\em Proceedings Of The 19th International Conference On World Wide Web}. (2010)
\bibitem{liau2018stochastic}Liau, D., Song, Z., Price, E. \& Yang, G. Stochastic multi-armed bandits in constant space. {\em International Conference On Artificial Intelligence And Statistics}. (2018)
\bibitem{liu2018transferable}Liu, B., Wei, Y., Zhang, Y., Yan, Z. \& Yang, Q. Transferable contextual bandit for cross-domain recommendation. {\em Proceedings Of The AAAI Conference On Artificial Intelligence}. (2018)
\bibitem{maiti2021multi}Maiti, A., Patil, V. \& Khan, A. Multi-Armed Bandits with Bounded Arm-Memory: Near-Optimal Guarantees for Best-Arm Identification and Regret Minimization. {\em Advances In Neural Information Processing Systems}. \textbf{34} (2021)
\bibitem{perrot2015theoretical}Perrot, M. \& Habrard, A. A theoretical analysis of metric hypothesis transfer learning. {\em International Conference On Machine Learning}. (2015)
\bibitem{ras2021recommender}Ras, Z., Wieczorkowska, A. \& Tsumoto, S. Recommender Systems for Medicine and Music.  (2021)
\bibitem{robbins1952some}Robbins, H. Some aspects of the sequential design of experiments. {\em Bulletin Of The American Mathematical Society}. (1952)
\bibitem{russac2019weighted}Russac, Y., Vernade, C. \& Cappé, O. Weighted linear bandits for non-stationary environments. {\em Advances In Neural Information Processing Systems}. (2019)
\bibitem{soare2014multi}Soare, M., Alsharif, O., Lazaric, A. \& Pineau, J. Multi-task linear bandits. {\em NIPS2014 Workshop On Transfer And Multi-task Learning: Theory Meets Practice}. (2014)
\bibitem{stark2019literature}Stark, B., Knahl, C., Aydin, M. \& Elish, K. A literature review on medicine recommender systems. {\em International Journal Of Advanced Computer Science And Applications}. (2019)
\bibitem{suk2021self}Suk, J. \& Kpotufe, S. Self-Tuning Bandits over Unknown Covariate-Shifts. {\em Algorithmic Learning Theory}. (2021)
\bibitem{thompson1933likelihood}Thompson, W. On the likelihood that one unknown probability exceeds another in view of the evidence of two samples. {\em Biometrika}. (1933)
\bibitem{tommasi2014learning}Tommasi, T., Orabona, F. \& Caputo, B. Learning Categories From Few Examples With Multi Model Knowledge Transfer. {\em IEEE Transactions On Pattern Analysis And Machine Intelligence}. (2014)
\bibitem{xu2021memory}Xu, X. \& Zhao, Q. Memory-Constrained No-Regret Learning in Adversarial Multi-Armed Bandits. {\em IEEE Transactions On Signal Processing}. (2021)
\bibitem{yang2007cross}Yang, J., Yan, R. \& Hauptmann, A. Cross-domain video concept detection using adaptive svms. {\em Proceedings Of The 15th ACM International Conference On Multimedia}. (2007)
\bibitem{zhao2014online}Zhao, P., Hoi, S., Wang, J. \& Li, B. Online Transfer Learning. {\em Artificial Intelligence}. (2014)
\bibitem{zhao2020handling}Zhao, P., Cai, L. \& Zhou, Z. Handling concept drift via model reuse. {\em Machine Learning}. (2020)
\bibitem{zhou2017large}Zhou, Q., Zhang, X., Xu, J. \& Liang, B. Large-scale bandit approaches for recommender systems. {\em International Conference On Neural Information Processing}. (2017)
\end{thebibliography}
\end{document}